\def\year{2017}\relax
\documentclass[letterpaper]{article}
\usepackage{aaai17}
\usepackage{times}
\usepackage{helvet}
\usepackage{courier}
\usepackage{graphicx,amssymb,amstext,amsmath, amsthm}
\usepackage{subfig}
\usepackage{algorithm}
\usepackage{algpseudocode}
\usepackage{framed}
\usepackage{wrapfig}

\usepackage{color}

\newcommand{\B}[1]{\textbf{#1}}
\newcommand{\Bu}{\mathcal{B}}

\renewcommand{\=}{\!=\!}

\newcommand{\bs}{\boldsymbol}
\newcommand{\mc}{\mathcal}

\newcommand{\T}{\mathcal{T}}
\newenvironment{myproof}[1][\proofname]{\proof[#1]\mbox{}}{\endproof}

\newcommand{\lb}{\underline}
\newcommand{\ub}{\overline}

\DeclareMathOperator{\mr}{MR}	

\newcommand{\feas}{\mc{X}}      
\DeclareMathOperator{\mrr}{MRR} 
\DeclareMathOperator*{\argmax}{argmax}
\DeclareMathOperator{\mincost}{mc}

\newcommand\Mark[1]{{\textsuperscript{\textnormal{#1}}}}

\newtheorem{thm}{Theorem}
\newtheorem{cor}{Corollary}

\newtheorem{ampn}{Assumption}
\newtheorem{lemma}{Lemma}
\begin{document}
\title{Robust Optimization for Tree-Structured Stochastic Network Design}
\author{Xiaojian Wu\Mark{1} ~~~~ Akshat Kumar\Mark{2} ~~~~ Daniel Sheldon\Mark{3,4} ~~~~ Shlomo Zilberstein\Mark{3} \\[4pt]
  \Mark{1} Department of Computer Science, Cornell University, USA \\
   \Mark{2} School of Information Systems, Singapore Management University, Singapore \\
  \Mark{3} College of Information and Computer Sciences, University of Massachusetts Amherst, USA\\
  \Mark{4} Department of Computer Science, Mount Holyoke College, USA \\
   {\tt xw458@cornell.edu} ~~~~~ {\tt akshatkumar@smu.edu.sg} ~~~~~ {\tt \{sheldon,shlomo\}@cs.umass.edu}
}

\maketitle
\begin{abstract}
Stochastic network design is a general framework for optimizing network connectivity. It has several applications in computational sustainability including spatial conservation planning, pre-disaster network preparation, and river network optimization. A common assumption in previous work has been made that network parameters (e.g., probability of species colonization) are precisely known, which is unrealistic in real-world settings. We therefore address the \textit{robust river network design problem} where the goal is to optimize river connectivity for fish movement by removing barriers. We assume that fish passability probabilities are known only imprecisely, but are within some interval bounds. We then develop a planning approach that computes the policies with either high \textit{robust ratio} or low \textit{regret}. Empirically, our approach scales well to large river networks.
We also provide insights into the solutions generated by our robust approach, which has significantly higher robust ratio than the baseline solution with mean parameter estimates.
\end{abstract}

\section{Introduction}
Many problems, such as influence maximization~\cite{Kempe03}, spatial and fish conservation planning~\cite{Sheldon10,Hanley05}, and predisaster preparation~\cite{schichl2015predisaster} can be formulated as a variant of the stochastic network design problem.
A stochastic network design problem (SNDP) is defined by a directed graph where each edge is either present or absent with some probability.
Management actions can be taken to change the probabilities of edge presence.
The goal is to determine which actions to take, subject to a budget, to optimize some outcome of the stochastic network over a time period.
Several approaches to solve SNDPs have been 
shown to scale up to large networks~\cite{chen2010scalable,Kumar12,wu2014stochastic,wu2016optimizing}.

An important assumption made in SNDPs is that the network parameters (e.g., probabilities of edge presence) are estimated accurately, which is not feasible in real world ecological domains due to noisy observations, model drift, climate change, and the diversity of species. To handle parameter uncertainty, researchers have formulated \textit{robust} network design problems that include uncertain network probabilities~\cite{he2014stability,chen2016scalable}.
Recently,~\citeauthor{kumar2016robust}~(2016) also studied a robust conservation planning problem where the movement probabilities of species and sizes of habitats are not accurately specified. The robust network design problem we address differs from previous work, which does not allow  management actions to modify interval parameters (e.g., edge probabilities). They only modify network structure, for example, by adding sources or nodes. In contrast, we allow management actions that can modify both interval bounds and network structure. As a result of the richer action space, 
it is unclear whether the sample average approximation (SAA) approach used in previous settings~\cite{kumar2016robust} is applicable to our problem.
To address these challenges, we develop a dynamic programming and mixed-integer programming based approach that can optimize connectivity without using SAA.

We study robust SNDPs for tree-structured river networks.
The motivating application is the barrier removal problem~\cite{neeson2015enhancing}, where the goal is to decide which instream barriers to remove or repair to help fish move upstream and get access to their historical habitats.
In this domain, the passage probability of a barrier can only be inaccurately estimated, and the new passage probability of a repaired barrier is even harder to estimate.
Hence, we model the uncertainty in passage probability using well known interval bounds~\cite{boutilier2003constraint}. We then develop a scalable algorithm to find the \emph{robust policy} for barrier removal.

The robustness of a policy can be quantified by two correlated metrics: \emph{robust ratio}~\cite{he2014stability,chen2016scalable} and \emph{regret}~\cite{boutilier2003constraint,kumar2016robust}.
Intuitively, assume that \emph{given} a policy, nature chooses an \textit{adversarial policy} that selects parameters within their interval bounds so as to either minimize the ratio between the values of the given policy and the adversarial policy (called \textit{robust ratio}) or maximize the value difference between them (called \textit{regret}).
We develop a scalable algorithm to find a robust policy that maximizes the robust ratio by solving a bilevel optimization problem.
We also show that, with minor modifications, our approach can be used to minimize regret.


The algorithm is based on a constraint generation procedure~\cite{boutilier2003constraint} that interleaves between two optimization steps.
The \textit{decision optimization} step finds a decision policy that maximizes the robust ratio when nature can choose policies and probabilities from a given \textit{limited} number of choices. In the second \textit{ratio minimization} step, the best adversarial policy and probabilities are found for the selected decision policy and are added to the set of choices for nature.
We provide a mixed integer linear programming formulation for the \textit{decision optimization} problem.
The \textit{ratio minimization} problem is much harder; we develop an algorithm called rounded dynamic programming (RDP) by combining a dynamic programming algorithm and a rounding method and show that it is a fully polynomial time approximation schema (FPTAS).
In experiments, we show that RDP performs nearly optimally as it selects the adversarial policy and probabilities.
Our algorithm can find policies that are more robust than policies found by baseline methods with respect to both robustness metrics.
We also provide insights on the robustness metrics by visualizing the solutions.

\section{River Network Design}
The problem is defined on a directed rooted tree $\T\=(V, E)$ with a unique \emph{root} denoted by $s$.
Edges spread out from the root. A node $v$ represents a contiguous region of the river network. It denotes a connected set of stream segments among which fish can move freely without passing any barriers. A node
$v$ is associated with a reward $r_v$ which is proportional to the total amount of habitat in that region (e.g., the total length of all segments).
An edge $e$ encodes a river barrier.
Fig.\,\ref{fig:river2tree} shows how to encode a river network as a directed rooted tree.
\begin{figure}[t]
\centering
\subfloat[River segments]{\includegraphics[height=32mm]{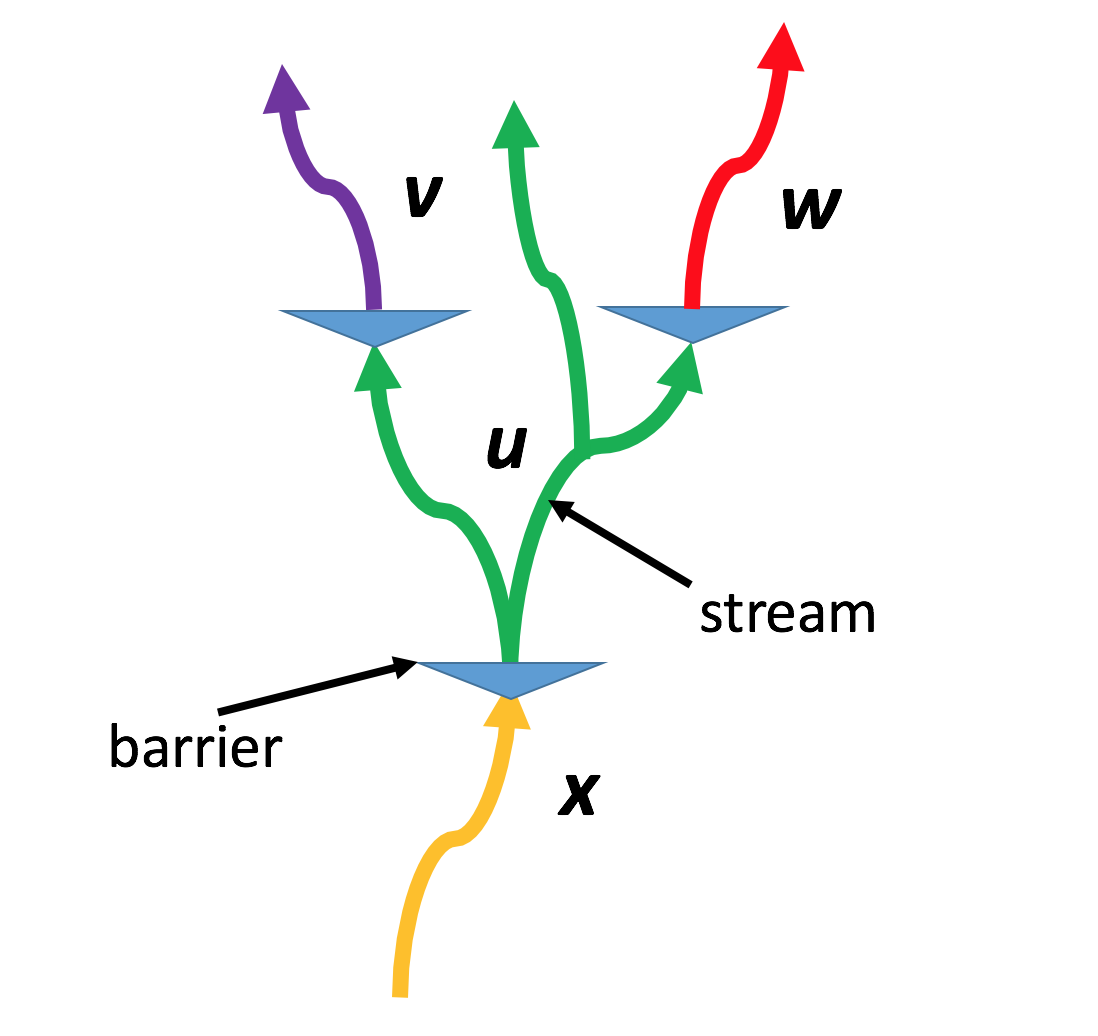}}
\subfloat[Directed rooted tree]{\includegraphics[height=32mm]{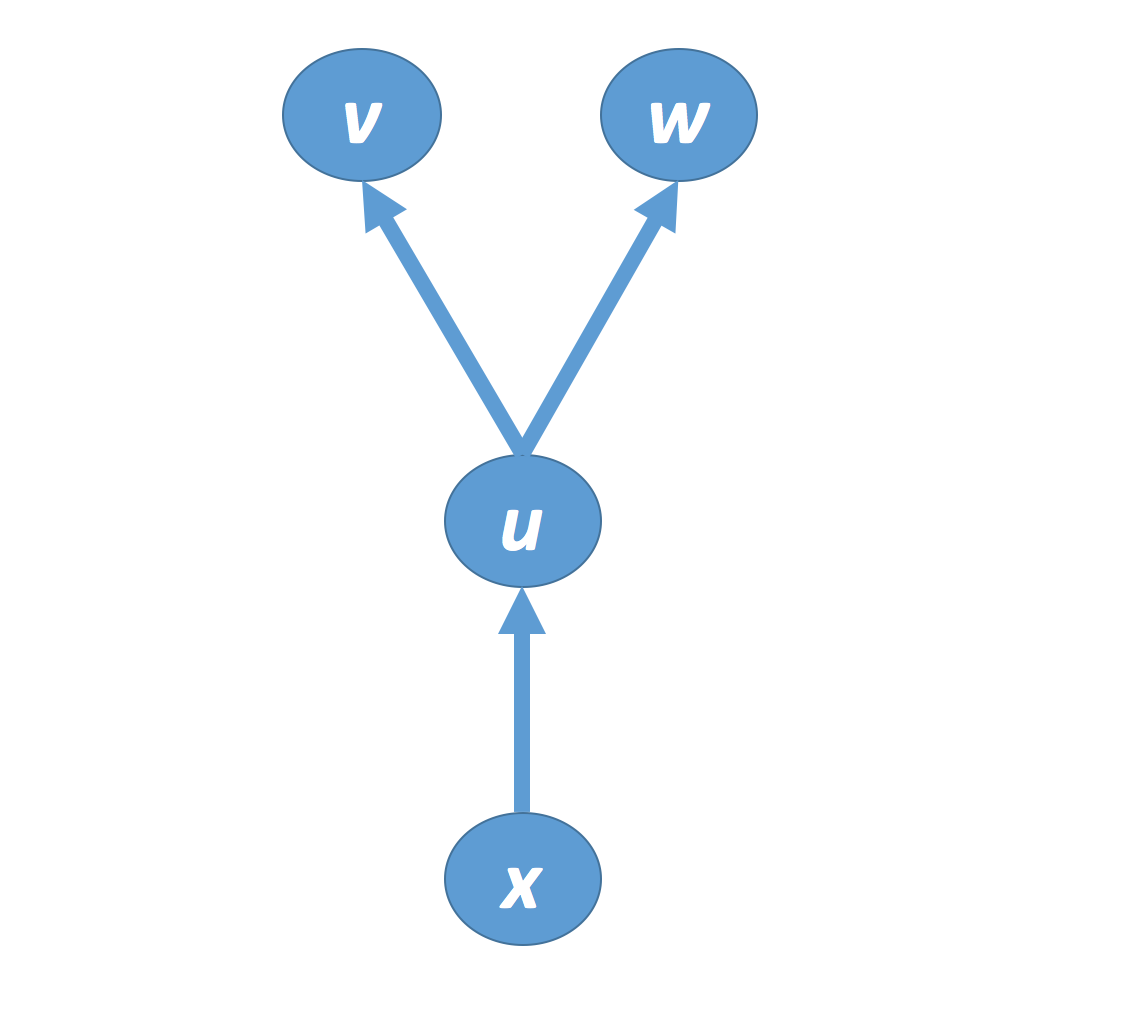}}
\caption{Encoding a river network as a directed rooted tree. Each color represents a contiguous region.}
\label{fig:river2tree}
\end{figure}
Each barrier is associated with a \emph{passage probability}---the probability that a fish can pass the barrier. Before any repair action is taken, the probability is called the \emph{initial passage probability} denoted by $p_e$.
A finite set of \emph{candidate actions} denoted by $A_e \= \{0, 1,...,m\}$ are available at $e$; an action $i$ has cost $c_e(i)$, and, if taken, can raise passage probability to $p_{e | i}$.
The action $0$ is the \emph{null action} with $p_{e|0} \= p_e$ and zero cost. 
A policy $\pi$ indicates which action is taken at each edge.
The passage probability for a given policy is denoted by $p_{e | \pi}$.
The accessibility of a node $v$ denoted by $p_{s \leadsto v | \pi}$ is the probability that a fish passed all barriers on the path from $s$ to $v$ or $p_{s \leadsto v | \pi} \= \prod_{e: \text{ on path from $s$ to $t$}} p_{e | \pi}$.
A reward $r_v$ can be collected only if a fish can reach $v$.
The \emph{value} of policy $\pi$, denoted by $z(\pi)$, is the total reward of nodes weighted by their accessibilities: $z(\pi) = \sum_{v\in V} p_{s \leadsto v | \pi} r_v$.
We also call $z(\pi)$ the  \emph{objective value} to differentiate between other values assigned to $\pi$.
The barrier removal problem~\cite{wu2014rounded} is to find a policy maximizing $z(\pi)$ subject to a budget constraint:
\begin{align}
\label{eq:maxex}\arg\max_{\pi} z(\pi) ~~~~~s.t.~~~~~~ c(\pi) \leq \Bu
\end{align}
where $c(\pi)$ is the total cost of action taken for each edge in the network. Let $\mc{X} = \{\pi: c(\pi) \leq \Bu\}$ denote the set of feasible policies. 

\begin{figure}[t]
\centering
\includegraphics[height=40mm]{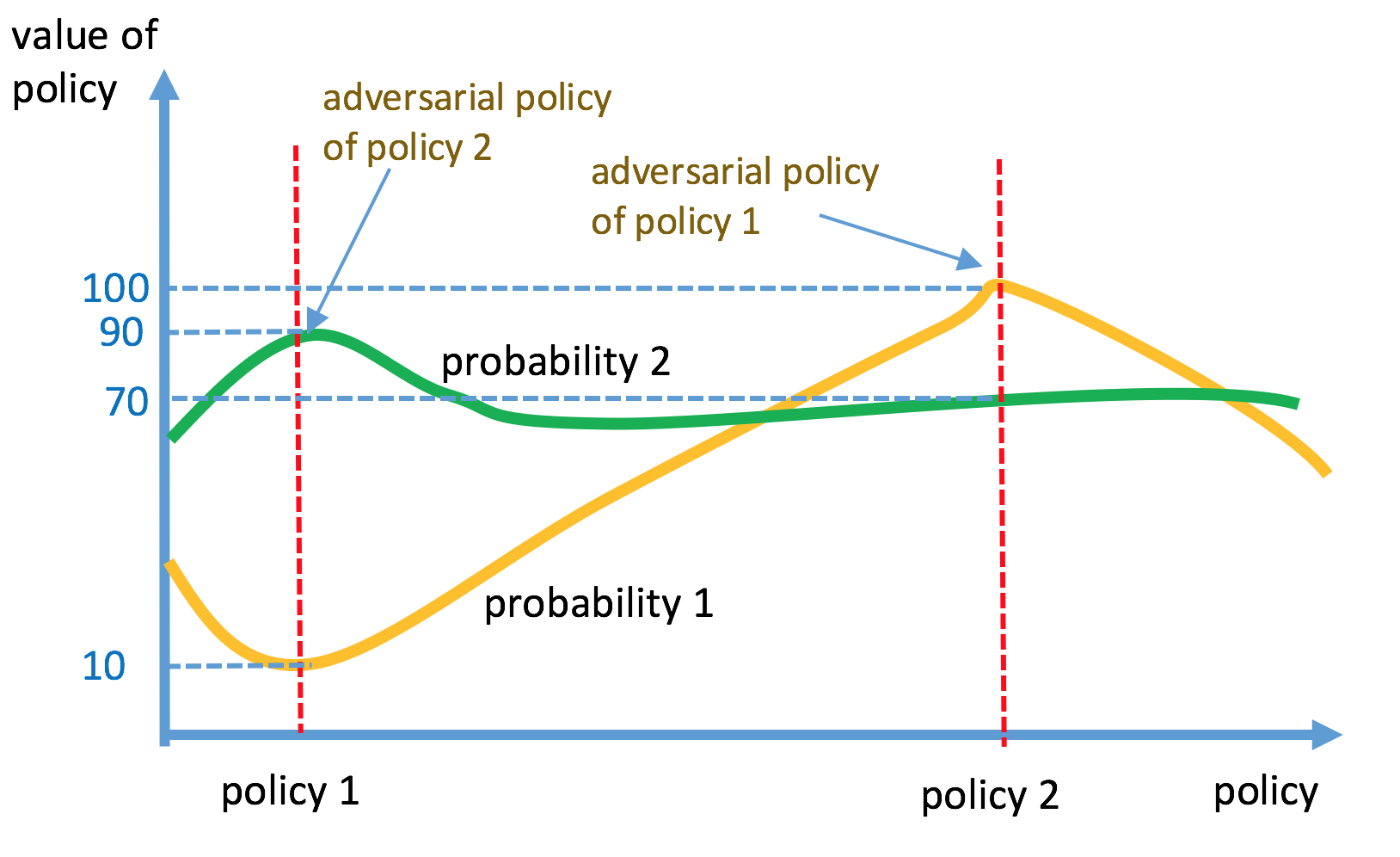}
\caption{Illustration of robust ratio with X-axis showing different policies. For {\small\sf policy\,1}, the adversary chooses {\small\sf policy\,2} and {\small\sf probability 1} (yellow curve) to minimize the robust ratio, which is $0.1$.
Similarly, the robust ratio of {\small\sf policy\,2} is $\frac{7}{9}$, hence it is more robust than {\small\sf policy\,1}.}
\label{fig:illustration}
\end{figure}

\subsubsection{Robust River Network Design} 
The barrier removal problem is defined upon the assumption that all the passage probabilities are known. However, this is an unrealistic assumption. Often, in real world settings, it is not possible to accurately estimate such probabilities. 
Therefore, in our model only interval bounds are specified for different probabilities~\cite{boutilier2003constraint}.  
Specifically, the passage probability for an edge $e$ and action $i \!\in\! A_e$ can take any value within a given interval. That is, $p_{e|i} \!\in\! \mc{P}_{e|i} \= [\underline{p}_{e|i}, \overline{p}_{e|i}]$.
Let $\bs{p}$ denote a vector of all probabilities $\bs{p} \= (p_{e|i})_{e\in E, i\in A_e}$. Let the space of all the allowed probabilities $\bs{p}$ be denoted as $\mc{P} \= \times_{e\in E, i\in A_e} \mc{P}_{e|i}$.
Our goal is to find a policy $\pi^{\mrr}$ that \emph{maximizes the robust ratio} as defined by 
\citeauthor{kouvelis2013robust} (2013) and \citeauthor{chen2016scalable} (2016):
\begin{equation}
\label{eq:mmratio}
\pi^{\mrr} \in \argmax_{\pi \in \feas} \min_{\pi' \in \feas, \bs{p} \in \mc{P}} \frac{z(\pi; \bs{p})}{z(\pi'; \bs{p})}.
\end{equation}
In the outer maximization, the decision maker seeks a \emph{decision policy} $\pi$ that is robust relative to adversarial choices made by nature. In the inner minimization, nature adversarially chooses a policy $\pi'$ and feasible parameters $\bs{p}$ (a \emph{policy-parameter pair}) to minimize the ratio between the value of the decision policy $\pi$ and the adversarial policy $\pi'$ on this set of parameters. The optimal value of the adversary is called the \emph{robust ratio} of policy $\pi$ with respect to parameter space $\mc{P}$. A policy (such as $\pi^{\mrr}$) that maximizes the robust ratio is called \emph{MRR-optimal}, and the robust ratio of such a policy is called the \emph{MRR-value}. Suppose $\pi^{\mrr}$ is MRR-optimal with MRR-value $\alpha$: then $\pi^{\mrr}$ achieves at least $\alpha$ fraction of the optimal reward for any parameter setting $\bs{p} \in \mc{P}$.
Fig.\,\ref{fig:illustration}  illustrates the concept.

\section{Our Method}
We develop an iterative method (Algorithm\,\ref{algorithm}) to solve Problem\,(\ref{eq:mmratio}) using constraint generation~\cite{boutilier2003constraint}.
\begin{algorithm}[t] \small
\caption{Robust Policy Optimization}
\label{algorithm}
\begin{algorithmic}[1]
\State Initialize $C = \{(\pi_0^\prime, \bs{p}_0)\}$ and  $T = 1$. 
\State \B{Decision Optimization:} obtain $\pi_T$ by solving:
\begin{align}
\label{eq:step2} U = \max_{\pi} \min_{(\pi^\prime, \bs{p}) \in C} z(\pi; \bs{p})  / z(\pi^\prime; \bs{p})
\end{align}
\State \B{Adversary Optimization:} obtain the adversarial policy-parameter pair $(\pi_{T}^\prime, \bs{p}_{T})$ with respect to $\pi_T$ by solving:
\begin{equation}
\label{eq:step3}
L= \min_{(\pi^\prime, \bs{p}) \in C} z(\pi_T; \bs{p}) / z(\pi^\prime; \bs{p}).
\end{equation}
\State if $U - L \leq threshold$, return $\pi_T$. Otherwise set $C = C\cup \{(\pi_{T}^\prime, \bs{p}_{T})\}$, increment $T$, and go to step 2.
\end{algorithmic}
\end{algorithm}
The high-level idea is to interleave two optimization problems. First, in the \emph{decision optimization problem}, the decision maker finds the best decision policy $\pi_T$ relative to a \emph{limited} adversary, who can only pick policy-parameter pairs from the finite set $C$. Then, the adversary selects a new policy-parameter pair to minimize the robust ratio with respect to the current decision policy $\pi_T$. The decision player's value $U$ is an upper bound on the MRR-value, because the adversary is limited to a finite subset of policy-parameter pairs. The adversary's optimal value $L$ is a lower bound on the MRR-value. When $U = L$, we have an MRR-optimal decision policy. By allowing a small gap between the two bounds, we can find a nearly MRR-optimal policy. The set $C$ is initialized with an arbitrary policy and probabilities.

\subsection{The Decision Optimization Problem}
The goal of Problem~\eqref{eq:step2} is to find a decision policy that maximizes the robust ratio with respect to the limited adversary. Fig.\,\ref{fig:MIP} presents a mixed-integer linear program (MILP) to solve this problem building on techniques from~\cite{neeson2015enhancing}. The variable $M$  encodes the MRR-value.
The inner minimization is replaced by inequality constraints~(\ref{MIP:CM}) on $M$. 
The continuous variable $z_{\bs{p}}$ encodes the objective value of the decision policy for probability setting $\bs{p}$ by~(\ref{MIP:Cer}). 
$z(\pi^\prime; \bs{p})$ is a constant for each policy-parameter pair $(\pi'; \bs{p}) \in C$.
$x^i_e$ is a binary decision variable indicating whether action $i \in A_e$ is applied to $e$ ($=1$) or not ($=0$). 
Constraint~(\ref{MIP:Cdecion}) enforces that one and only one action is taken at each edge, and~(\ref{MIP:budget}) is the budget constraint.

The constraint set $\Omega(\bs{p},x)$ defined in (\ref{MIP:er})--(\ref{MIP:18}) forces $z_{\bs{p}}$ to be the objective value of $\pi$ under probability setting $\bs{p}$.
 The variable $\alpha^{\bs{p}}_v$ encodes the accessibility of node $v$. 
The root node has accessibility $1$ by (\ref{MIP:root}). 
$\Pi(v)$ denotes the \emph{parent} of node $v$. 
Recall that each node has at most one parent.
The variable $\lambda^{\bs{p}}_{v,i}$ encodes the increment in the accessibility of node $v$ if an action $i \in A_{\Pi(v), v}$ is applied to edge $(\Pi(v), v)$.
In~(\ref{MIP:access}), the accessibility of $v$ equals to the cumulative passability when no action is taken on edge $(\Pi(v), v)$ (the term $\alpha^{\bs{p}}_{\Pi(v)} p_{\Pi(v)v}$) plus the total increment (the term $\sum_{i\in A_{\Pi(v)v}}\lambda^{\bs{p}}_{v, i} $).
Actually, at most one action can be taken, so only one $\lambda^{\bs{p}}_{v, i}$ will be nonzero in the summation.
The increment $\lambda^{\bs{p}}_{v,i}$ is nonzero only if $x^i_{\Pi(v)v}$ is $1$ by~(\ref{MIP:lambda}), and can be at most $(p_{\Pi(v)v | i} - p_{\Pi(v)v}) ~ \alpha^{\bs{p}}_{\Pi(v)}$ by~(\ref{MIP:lambda2}), which is exactly the increment when action $i$ is taken.

\begin{figure}[t]
\begin{framed}
\vspace{-14pt}
{\scriptsize
\begin{align}
& \max M \\
\label{MIP:CM}& M \leq \frac{z_{\bs{p}} }{z(\pi^\prime; \bs{p})} ~~~~~~~~~~~~~~~~~~~~~~~~~~~~~~~~~~~~\forall (\pi^\prime; \bs{p}) \in C \\
\label{MIP:Cer}& z_{\bs{p}} \in \Omega(\bs{p}, x) ~~~~~~~~~~~~~~~~~~~~~~~~~~~~~~~~~~~~~~~\forall (\pi^\prime; \bs{p}) \in C \\
\label{MIP:Cdecion}& \sum_{i \in A_e} x^i_e = 1 ~~~~~~~~~~~~~~~~~~~~~~~~~~~~~~~~~~~~~~~~\forall e \in E \\
\label{MIP:budget}& \sum_{e\in E} \sum_{i \in A_e} c_i x_i \leq \Bu  \\
\label{MIP:decision} & x^i_e \in \{0, 1\} ~~~~\forall e \in E, \forall i \in A_e \\
& \text{\B{Constraint set}}~\Omega(\bs{p}, x) \\
\label{MIP:er}& z_{\bs{p}} = \sum_{v\in V} \alpha^{\bs{p}}_v r_v   \\
\label{MIP:root}& \alpha^{\bs{p}}_s = 1   \\
\label{MIP:access}& \alpha^{\bs{p}}_v = \alpha^{\bs{p}}_{\Pi(v)} p_{\Pi(v)v} + \sum_{i\in A_{\Pi(v)v}}\lambda^{\bs{p}}_{v, i} ~~~~~~~~~\forall v \in V / \{s\} \\
\label{MIP:lambda}& \lambda^{\bs{p}}_{v, i} \leq x^i_{\Pi(v)v} ~~~~~~~~~~~~~~~~~~~~~~~~~~~~~~~~~~~~~~\forall v\in V / \{s\}, \forall i \in A_{\Pi(v)v} \\
\label{MIP:lambda2}& \lambda^{\bs{p}}_{v, i} \leq (p_{\Pi(v)v | i} - p_{\Pi(v)v}) ~ \alpha^{\bs{p}}_{\Pi(v)}  ~ \forall v\in V / \{s\}, \forall i\in A_{\Pi(v)v}\\
\label{MIP:17}& \alpha^{\bs{p}}_v \in [0, 1] ~~~~~~~~~~~~~~~~~~~~~~~~~~~~~~~~~~~~~~~~~~~~~\forall (\pi^\prime, \bs{p}) \in C, \forall v\in V \\
\label{MIP:18}& \lambda^{\bs{p}}_{e,i} \in [0,1] ~~~~~~~~~~~~~~~~~~~~~~~~~~~~~~\forall (\pi^\prime, \bs{p}) \in C, \forall e\in E, \forall i \in A_e
\end{align}
}
\vspace{-14pt}
\end{framed}
\caption{Mixed integer linear program to maximize the robust ratio for a given set $C$}
\vspace{-2pt}
\label{fig:MIP}
\end{figure}

\subsection{The Adversary Optimization Problem}
In the adversary optimization step, we wish to solve Problem~\eqref{eq:step3} to find a policy-parameter pair $(\pi^{\prime*}, \bs{p}^*)$ to minimize the robust ratio with respect to the current decision policy.

Here is our main result.
\begin{thm}
There is an FPTAS for problem (\ref{eq:step3}). It finds a policy-parameter pair with robust ratio at most $(1 + \epsilon) OPT$ in time $O(\frac{n^4}{\mu^2})$ where $\mu = \frac{\epsilon}{2 + \epsilon}$, $n$ is the number of nodes in the tree, and $OPT$ is the optimal value of (\ref{eq:step3}).
\label{thm:main}
\end{thm}
\noindent
The FPTAS only approximately minimizes the objective, so the value $\hat{L}$ it achieves not a lower bound in  in Algorithm~\ref{algorithm}. However, the approximation guarantee implies that $L = \frac{\hat{L}}{1 + \epsilon}$ is a lower bound.

In the rest of this section, we prove Theorem~\ref{thm:main} (proofs of some auxiliary results are left in appendix). We first propose a dynamic programming (DP) algorithm for problem~(\ref{eq:step3}), but this takes exponential time. We then develop a rounding strategy to reduce the running time to polynomial time and prove that this is an FPTAS. 
This basic idea is originally used for the barrier removal problem~(\ref{eq:maxex})~\cite{wu2014rounded}.
The adversary optimization problem here is more complex as the adversary tries to simultaneously minimize the value of decision policies and maximizes the value of adversarial policies.
To guarantee the approximation rate, we round these two values distinctly.

To simplify the presentation, we assume without loss of generality the following:
\begin{ampn}
Each node $u\in \T$ has at most two children.
\end{ampn}
\noindent Any problem instance can be converted to satisfy this assumption~\cite{wu2014rounded}.
Our first lemma restricts the space of parameters to be considered.
\begin{lemma}
\label{lemma:probsetting} 
There exists an optimal policy-parameter pair $(\pi^{\prime*}, \bs{p}^*)$ for Problem~(\ref{eq:step3}) with the following property.
Suppose $\pi^{\prime*}$ takes action $i$ and the decision policy $\pi$ takes action $j$ on edge $e$.
If $j\neq i$, then $p^*_{e|i} = \ub{p}_{e|i}$ and $p^*_{e|j} = \lb{p}_{e|j}$.
Otherwise, $p^*_{e|i}$ is either $\lb{p}_{e|i}$ or $\ub{p}_{e|i}$.
\end{lemma}
\noindent Lemma~\ref{lemma:probsetting} guarantees that the optimal adversary probability is either the upper or lower bound of the interval.

\subsubsection{Policy-Parameter Actions and Optimization}
First, we redefine problem~(\ref{eq:step3}) in the following way so that it is amenable to dynamic programming.

Let $\pi$ be fixed.
The new optimization problem is the same as the river network design problem~(\ref{eq:maxex}) except that its objective is the robust ratio $\frac{z(\pi; \bs{p}) }{z(\pi^\prime; \bs{p})}$ and its actions encode both the actions and parameters of the adversary.

We define a finite set of \emph{policy-parameter actions} $A^p_e$ for each edge, which encode choices made by the adversary for edge $e$, including both the action taken and the probability setting for each available action.
A policy-parameter action is a vector $(\bs{i}^a_e, \bs{p}_{e | 0}, ..., \bs{p}_{e | |A_e|})$ taking value in $A^p_e = A_e \times \prod_{j\in A_e} \{\lb{p}_{e|j}, \ub{p}_{e|j}\}$.
$\bs{i}^a_e$ specifies the action that the adversary takes at $e$.
$\bs{p}_{e | j}$ specifies the passage probability on $e$ for action $j$.
It is easy to see from Lemma~\ref{lemma:probsetting} that a given policy-parameter action need only consider $\lb{p}_{e|j}$ and $\ub{p}_{e|j}$ as possible values for $\bs{p}_{e|j}$ without sacrificing optimality.
In addition, Lemma~\ref{lemma:probsetting} allows us to eliminate certain policy-parameter actions from consideration.
For example, if $A_e = \{0, 1\}$ and the decision policy $\pi$ takes action $1$, $A^s_e$ only needs to include $3$ policy-parameter actions
{\small
\begin{align*}
(0, \ub{p}_{e | 0}, \lb{p}_{e | 1}), ~~  (1, \lb{p}_{e | 0}, \lb{p}_{e | 1}), ~~ (1, \lb{p}_{e | 0}, \ub{p}_{e | 1})
\end{align*}
}
More generally, we have
\begin{cor}
For a fixed $\pi$, only $|A_e| + 1$ actions in $A^s_e$ are needed to compute $(\pi^{\prime*}, \bs{p}^*)$.
\end{cor}

In summary, the choice of a policy-parameter action for each edge to minimize the robust ratio gives the optimal policy-parameter pair $(\pi^{\prime*}, \bs{p}^*)$ for problem~(\ref{eq:ineq3}).

\subsubsection{Dynamic Programming}
We now present a dynamic programming algorithm to solve this new problem with policy-parameter actions.

In a rooted directed tree, each node $u$ corresponds to a subtree $\T_u$.
Define $\pi_u$ (or $\pi^\prime_u$) to be the subset of $\pi$ (or $\pi^\prime$) that only includes actions for edges within $\T_u$, and define $\bs{p}_u$ to be the subset of $\bs{p}$ including probabilities only in $\T_u$.
Define $z_u(\pi_u; \bs{p}_u)$ to be the objective value of policy $\pi_u$ on subtree $\T_u$ with probability vector $\bs{p}_u$ pretending that $u$ is the overall root, i.e., $z_u(\pi_u; \bs{p}_u) = \sum_{t\in \T_u} p_{u\leadsto t | \pi} r_t$.
Similarly, $z_u(\pi^\prime_u; \bs{p}_u)$ is the value of $\pi^\prime_u$ for $\T_u$.
The following recurrences calculate both values for a given $(\pi_u; \bs{p}_u)$
{\small \begin{align}
\label{eq:rec_a}& z_u(\pi^\prime_u; \bs{p}_u) \!=\! r_u \!+\! p_{uv | \pi^\prime_u} z_v(\pi^\prime_v; \bs{p}_v) + p_{uw | \pi^\prime_u} z_w(\pi^\prime_w; \bs{p}_w) \\
\label{eq:rec_d}& z_u(\pi_u; \bs{p}_u) \!=\! r_u \!+\! p_{uv | \pi_u} z_v(\pi_v; \bs{p}_v) + p_{uw | \pi_u} z_w(\pi_w; \bs{p}_w) 
\end{align}}
The DP table of subtree $\T_u$ is indexed by pairs $(z^a_u, z^d_u)$, where $z^a_u$ represents an objective value of an adversary policy and $z^d_u$ represents an objective value of the (fixed) decision policy on that subtree. The table includes only pairs that are achievable by some probability vector $\bs{p}_u$ and adversary policy $\pi'_u$ for subtree $\T_u$, that is, $z^d_u = z_u(\pi_u; \bs{p}_u)$ and $z^a_u = z_u(\pi'_u, \bs{p}_u)$. 
Let $\Phi(z^a_u, z^d_u) = \{(\pi^\prime_u, \bs{p}_u)~ | ~ z_u(\pi^\prime_u; \bs{p}_u) = z^a_u, z_u(\pi_u; \bs{p}_u) = z^d_u\}$ be the set of all policy-parameter pairs that map to a pair of objective values $(z^a_u, z^d_u)$. 
For the entry of the table indexed by $(z^a_u, z^d_u)$, we record only the \emph{minimum-cost} adversary policy, and the minimum cost (denoted by $mc$) it achieves:
\begin{align}
\mincost(z^a_u, z^d_u) = \min_{(\pi^\prime, \bs{p}) \in \Phi(z^a_u, z^d_u)} c(\pi^\prime_u)
\end{align}

The DP tables for all subtrees can be calculated recursively from leaf nodes toward the root $s$ in the following way.
First, the table at a leaf node contains a single tuple with cost $0$ because the subtree contains only the leaf node.
Consider a node $u$ with two children $v$ and $w$.
We can build the DP table at $u$ if we have the DP tables of $v$ and $w$ by computing all achievable objective-value pairs at $u$ and their minimum costs. 
From each pair $(z^a_v, z^d_v)$ at $v$ and each pair $(z^a_w, z^d_w)$ at $w$, policy-parameter pairs $(\pi^\prime_v,\bs{p}_v)$ and $(\pi^\prime_w, \bs{p}_w)$ can be extracted.
For each policy-parameter action $(i^a_{uv}, \bs{p}_{uv})$ on edge $(u,v)$ and each policy-parameter action $(i^a_{uw}, \bs{p}_{uw})$ on edge $(u,w)$, a new pair  $(\pi^\prime_u, \bs{p}_u)$ at $u$ can be built, with which we can compute a pair $(z^a_u, z^d_u)$ using recurrences~(\ref{eq:rec_a}) and~(\ref{eq:rec_d}).
The cost of this new pair is
\begin{align}
\label{eq:combine_c} c(i^a_{uv}) + c(i^a_{uw}) + \mincost(z^a_v, z^d_v) + \mincost(z^a_w, z^d_w)
\end{align}
The same pair may be generated multiple times, but only the minimum cost is recorded.

Once all DP tables are computed, the optimal solution can be extracted from the table at $s$ by finding a tuple
\begin{align*}
(z^{a*}_s, z^{d*}_s) \in \arg\min_{\mincost(z^a_s, z^d_s) \leq \Bu} ~~ \frac{z^d_s}{z^a_s}
\end{align*}
The pair $(\pi^{\prime*}, \bs{p}^*)$ associated with the tuple minimizes the objective.

Unfortunately, the table size grows exponentially with the height of the node in the tree. 
We next introduce a rounding strategy to make the algorithm scalable.

\subsubsection{Rounding}
We define rounded value functions $\hat{z}_u(\pi^\prime; \bs{p})$ and $\hat{z}_u(\pi; \bs{p})$ for subtree $u$ and introduce the following recurrences for \emph{rounded} value functions: \\
{\small \begin{align}
\label{eq:rec_ra}& \hat{z}_u(\pi^\prime_u; \bs{p}_u) \!=\! K_u \!\! \left\lfloor \!\frac{r_u \!+ p_{uv | \pi^\prime_u} \hat{z}_v(\pi^\prime_v; \bs{p}_v) + p_{uw | \pi^\prime_u} \hat{z}_w(\pi^\prime_w; \bs{p}_w)}{K_u} \!\right \rfloor\\[4pt]
\label{eq:rec_rd}& \hat{z}_u(\pi_u; \bs{p}_u) \!=\! K_u \!\!\left \lceil \frac{r_u \!+ p_{uv | \pi_u} \hat{z}_v(\pi_v; \bs{p}_v) + p_{uw | \pi_u} \hat{z}_w(\pi_w; \bs{p}_w)}{K_u} \!\right \rceil
\end{align}}
where $K_u$ is an user defined \emph{rounding parameter}.
Intuitively, values are rounded and grouped into discrete bins, which reduces the number of pairs in the DP table.
The following theorem states that for any given policy-parameter pair, the rounded objective values are not too far from the true values.
\begin{thm}
Let $\mu > 0$.
If we set $K_u = \mu r_u$, for any $(\pi^\prime_u, \bs{p}_u)$ and any $\pi_u$, we have
{\small 
\begin{align}
\label{eq:ineq1}&z_u(\pi^\prime_u;\bs{p}_u) \!-\! \hat{z}_u(\pi^\prime_u;\bs{p}_u) \!\leq \!\!\sum_{t\in \T_u} p_{u\leadsto t | \pi'_u} K_t \!=\! \mu z_u(\pi^\prime_u;\bs{p}_u)  \\
\label{eq:ineq2}& \hat{z}_u(\pi_u;\bs{p}_u) \!-\! z_u(\pi_u;\bs{p}_u) \!\leq \!\!\sum_{t\in \T_u} p_{u\leadsto t | \pi_u} K_t \!=\! \mu z_u(\pi_u;\bs{p}_u) \\
\label{eq:ineq3}& z_u(\pi^\prime_u;\bs{p}_u) \geq \hat{z}_u(\pi^\prime_u;\bs{p}_u) \\
\label{eq:ineq4}& \hat{z}_u(\pi_u;\bs{p}_u) \geq z_u (\pi_u;\bs{p}_u)
\end{align}
}
\end{thm}
\begin{myproof}[Proof sketch]
Intuitively, in~(\ref{eq:rec_ra}), the floor rounding operation at a node $t$ reduces the value by at most $K_t$, which is discounted by probability $p_{u\leadsto t | \pi^\prime}$.
Therefore, we have (\ref{eq:ineq1}) and (\ref{eq:ineq3}).
In~(\ref{eq:rec_rd}), the ceiling rounding operation at a node $t$ introduces an increment bounded by $K_t$, which is discounted by $p_{u\leadsto t | \pi^\prime}$.
Therefore, we have (\ref{eq:ineq2}) and (\ref{eq:ineq4}).
\end{myproof}

The rounded dynamic programming (RDP) algorithm works the same as the DP algorithm except that instead of keeping a list of $(z^a_u, z^d_u)$  in the table of $u$, a list of \emph{rounded pairs} denoted by $(\hat{z}^a_u, \hat{z}^d_u)$ are kept, which are calculated by recurrences~(\ref{eq:rec_ra}) and ~(\ref{eq:rec_rd}).
Each rounded pair is associated with the minimum cost to achieve it and the correspondent policy-parameter pair.
Intuitively, since multiple $z^a_u$s (or $z^d_u$s) are rounded into the same  $\hat{z}^a_u$ (or $\hat{z}^d_u$), the size of the table is reduced. 
It can be shown that RDP can find 
\begin{align}
(\pi^{\prime r}, \bs{p}^r)\in \arg\min_{\pi^\prime, \bs{p}} \frac{\hat{z}(\pi; \bs{p})}{\hat{z}(\pi^{\prime}; \bs{p})}
\end{align}
We show that $(\pi^{\prime r}, \bs{p}^r)$ is a good approximation to the optimal policy-parameter pair $(\pi^{\prime *}, \bs{p}^*)$.
That is, it is within $(1 + \epsilon)$ optimal if $\mu$ is set properly.
Specifically, 
\begin{thm}
If $\mu = \frac{\epsilon}{2 + \epsilon}$, we have
\begin{align*}
OPT = \frac{z(\pi; \bs{p}^*)}{z(\pi^{\prime *}; \bs{p}^*)} \leq \frac{z(\pi; \bs{p}^r)}{z(\pi^{\prime r}; \bs{p}^r)} \leq (1 + \epsilon) OPT
\end{align*}
\label{thm:appro_rate}
\end{thm}
\begin{proof}
By~(\ref{eq:ineq1}) and (\ref{eq:ineq2}), for any $(\pi^\prime, \bs{p})$, we have
{\small \begin{align*}
\frac{\hat{z}(\pi; \bs{p})}{\hat{z}(\pi^\prime; \bs{p})} \leq \frac{(1 + \mu) z(\pi; \bs{p})}{(1 - \mu) z(\pi^\prime; \bs{p})} = (1 + \epsilon) \frac{z(\pi; \bs{p})}{ z(\pi^\prime; \bs{p})}
\end{align*}}
Since $(\pi^{\prime r}, \bs{p}^r)$ produces the minimum ratio for rounded value functions~(\ref{eq:rec_rd}) and~(\ref{eq:rec_ra}), we have
\begin{align*}
 \frac{\hat{z}(\pi; \bs{p}^r)}{\hat{z}(\pi^{\prime r}; \bs{p}^r)} \leq \frac{\hat{z}(\pi; \bs{p}^*)}{\hat{z}(\pi^{\prime *}; \bs{p}^*)} \leq (1 + \epsilon) \frac{z(\pi; \bs{p}^*)}{ z(\pi^{\prime *}; \bs{p}^*)}
\end{align*}
By~(\ref{eq:ineq3}) and (\ref{eq:ineq4}), we have
\begin{align*}
\frac{z(\pi; \bs{p}^r)}{z(\pi^{\prime r}; \bs{p}^r)} \leq \frac{\hat{z}(\pi; \bs{p}^r)}{\hat{z}(\pi^{\prime r}; \bs{p}^r)}
\end{align*}
Thus, the theorem is proved.
\end{proof}

\subsubsection{Runtime Analysis}
In Theorem~\ref{thm:appro_rate}, we see that the $K_u$ values of affect the approximation rate.
Now, we analyze the dependence of the RDP algorithm running time on these values.
First, we make the following assumption.
\begin{ampn}
There are two constants $m$ and $M$ independent of $|V|$ such that $m \leq r_u \leq M$ for all $u \in V$. 
\end{ampn}
The assumption is reasonable because rewards represent habitat areas of stream segments, which do not increase or decrease as the number of segments increases. 

Let the number of different values of $\hat{z}^a_u$ and $\hat{z}^d_u$ in the table at $u$ be $m^a_u$ and $m^d_u $. We have 
\begin{lemma}
If $K_u = \mu r_u$, we have
\begin{align*}
 m^a_u = O\Big(\frac{n_u}{\mu}\Big), ~~~~~ m^d_u = O\Big(\frac{n_u}{\mu}\Big)
\end{align*}
where $n_u$ is the number of nodes in subtree $\T_u$.
\label{lemma:num_values}
\end{lemma}
\begin{proof}
Since $\hat{z}(\pi^\prime_u; \bs{p}_u)$ is upper-bounded by $z(\pi^\prime_u; \bs{p}_u) \leq n_u \cdot M$, the number of different rounded values with $K_u$ is $m^a_u \leq \frac{n_u \cdot M}{K_u} \leq \frac{n_u \cdot M}{\mu m} = O(\frac{n_u}{\mu})$.
Similarly, $\hat{z}(\pi; \bs{p})$ is upper-bounded by $(1 + \mu) z(\pi_u; \bs{p}_u) \leq (1 + \mu)n_u M$, so $m^d_u = O(\frac{n_u}{\mu})$ as well.
\end{proof}

Define $T(n_u)$ to be the running time for subtree $u$, which is calculated by recurrence
\begin{align*}
T(n_u) &= O(m^a_v m^d_v m^a_w m^d_w) + T(n_v) + T(n_w)
\end{align*}
Together with Lemma \ref{lemma:num_values}, it can be shown that 
\begin{thm}
$T(n_u) = O(\frac{n_u^4}{\mu^2})$.
\label{thm:time}
\end{thm}
Thus, the running time of the RDP algorithm is $O(\frac{n^4}{\mu^2})$ where $n$ is the number of nodes in the directed rooted tree.
Combining Theorems~\ref{thm:appro_rate} and~\ref{thm:time}, Theorem~\ref{thm:main} is proved.

\section{Other Criterion of Robustness}
A slightly different way to quantify robustness is to use \emph{regret}~\cite{kumar2016robust,boutilier2003constraint}.
The policy that minimizes the regret is defined by 
\begin{align}
\label{eq:mmregret}\pi^{\mr} \in \arg\min_{\pi: c(\pi) \leq \Bu} \max_{\pi^\prime: c(\pi^\prime)\leq \Bu} z(\pi^\prime; \bs{p}) -  z(\pi; \bs{p})
\end{align}

The robust ratio and the regret are correlated as
\begin{align*}
\frac{z(\pi; \bs{p})}{z(\pi^\prime; \bs{p})} = 1 - \frac{z(\pi^\prime; \bs{p}) -  z(\pi; \bs{p})}{z(\pi^\prime; \bs{p})}
\end{align*}
The robust ratio is in some way the scaled version of the regret.
In experiments, we show that $\pi^{\mrr}$ also produces small regret compared to policies computed by other baseline methods.
Our algorithm with minor modifications can find a nearly optimal $\pi^{MR}$ empirically.

\section{Experiments}

We use data from the CAPS project~\cite{McGarigal11} for the river networks in Massachusetts and synthetically define the missing parameters from the data.
The data provides the point estimates of the initial passability probabilities.
We use the method in~\cite{kumar2016robust} to define the intervals of initial passage probabilities before taking actions.
The interval of an initial passage probability is $[p - \beta  p ,   p + \beta  p]$ where $p$ is an point estimate and $\beta$ is a parameter controlling the interval sizes.

\begin{figure}[t]
\vspace{-8pt}
\centering
\subfloat[\small RDP for robust ratio]{\includegraphics[height=32mm]{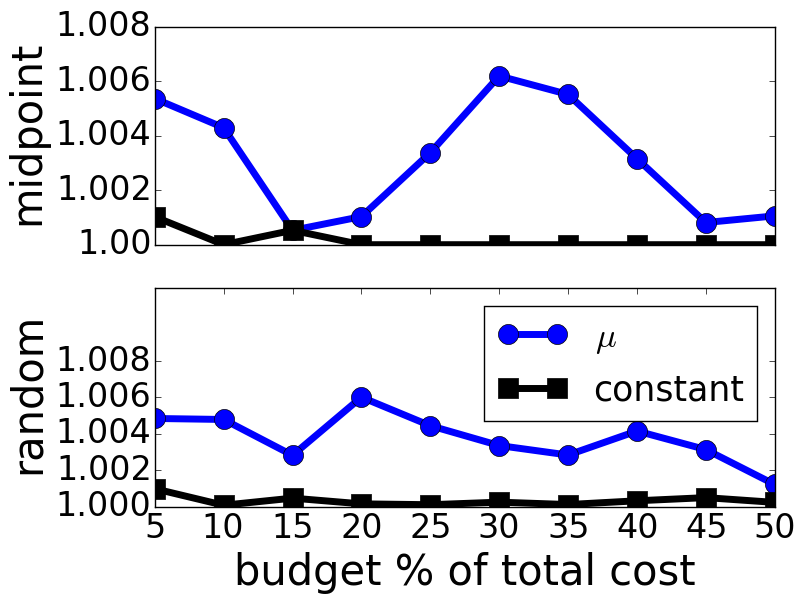}}
\subfloat[\small RDP for regret]{\includegraphics[height=32mm]{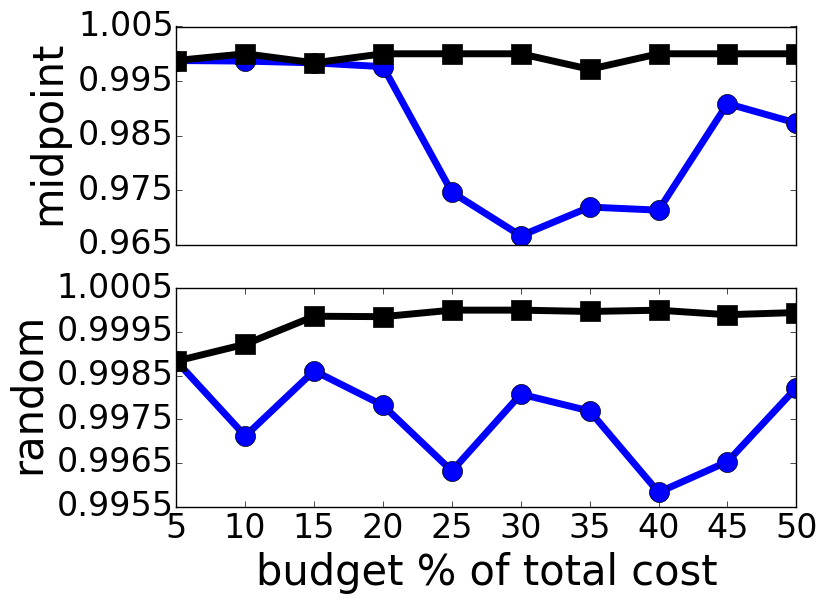}}
\caption{Approximate qualities for different algorithm configurations with $\beta = 0.3$. X-axis: budget sizes. Y-axis: $\frac{value}{OPT}$ where $OPT$ is the optimal value produced by the DP algorithm.
Value of random policies is an average of $10$ runs.}
\label{fig:appr_small}
\end{figure}

\begin{figure}[t]
\centering
\includegraphics[height=32mm]{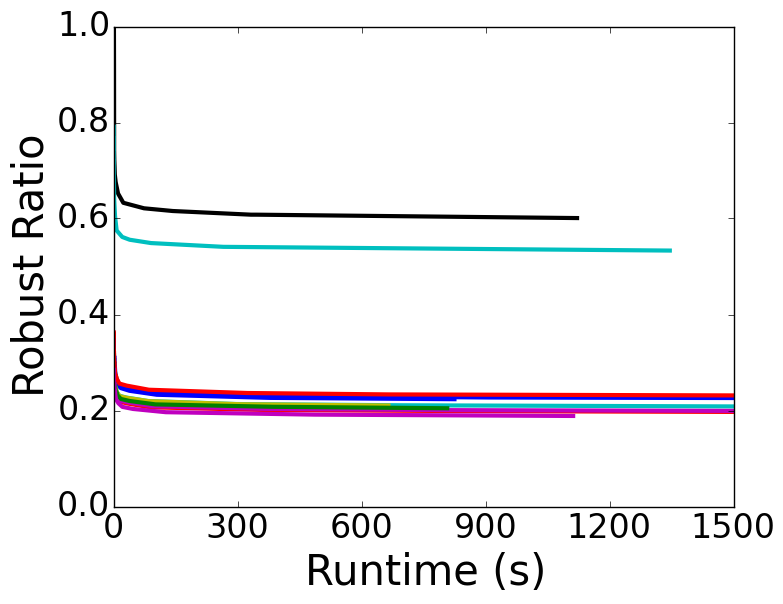}
\caption{Robust ratio for different $K$ values with $\beta {=} 0.3$. 
From top to bottom, curves are for ``midpoint" and 
``worst" policies, and $10$ random policies.}
\label{fig:appr_large}
\end{figure}

The data contains two types of barriers: culverts and dams.
The point estimates for culverts provided by the data are mostly in the range $[0.8, 0.9]$.
A typical action that removes a culvert raises its passage probability to $1.0$ and costs \$100,000.
Most of the point estimates for dams are less than $0.2$. 
A typical action to repair a dam costs \$173,030, and shifts its probability interval to $[p' - \beta  p' , p' + \beta  p']$ where $p' = p + \text{a random value in [0.5, 0.9]}$ .
The cost estimates are based on a study by~\citeauthor{neeson2015enhancing} (2015).
All intervals are truncated to fit within $[0, 1.0]$.

\begin{figure}[t]
\vspace{-8pt}
\centering
\subfloat[\small Robust Ratio]{\includegraphics[height=32mm]{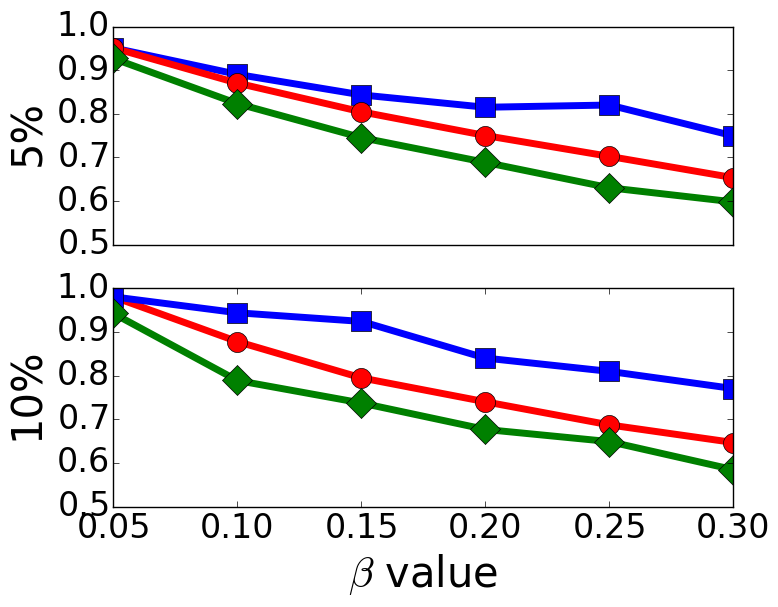}}
\subfloat[\small Regret]{\includegraphics[height=32mm]{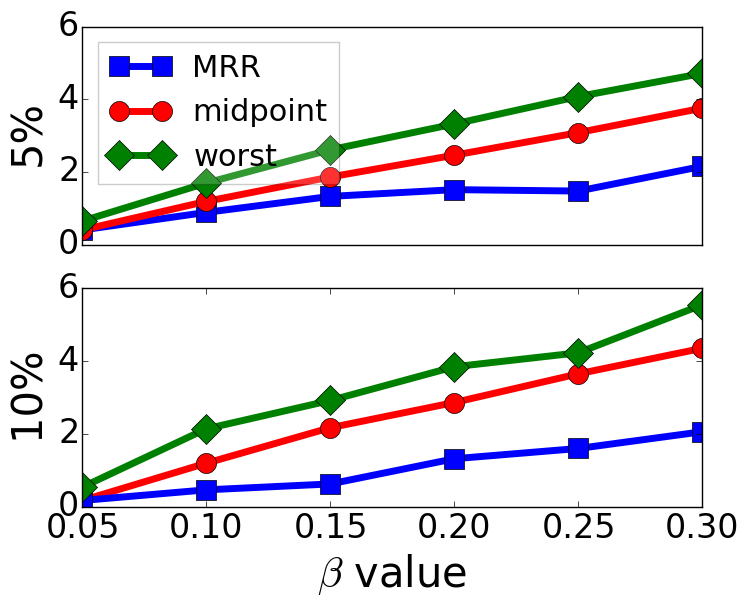}}
\caption{Robust ratio and regret ($\times 10^5$) for three type of policies under different $\beta$ and budget sizes of $5\%$ and $10\%$ .}
\label{fig:robust_ratio}
\end{figure}

We compare our algorithm against two baseline methods:
a ``midpoint" policy is obtained by solving problem~(\ref{eq:maxex}) and assuming true passage probabilities being the mid-point values of the intervals;
a ``worst" policy is obtained by solving problem~(\ref{eq:maxex}) and conservatively assuming true passage probabilities being the lower bounds of the intervals.
The policy calculated by our algorithm is the ``MRR" policy.

\subsubsection{Approximate Rate of the RDP Algorithm}
First, we evaluate the approximation rates of the RDP algorithm for problem~(\ref{eq:step3}), and of a \emph{modified RDP algorithm} for solving the inner maximization problem of~(\ref{eq:mmregret}) on a small network of only $22$ nodes.
The DP algorithm runs out of memory on networks of larger sizes. 
The results are shown in Fig.\,\ref{fig:appr_small}.
We set $K_u$ in two different ways---$\epsilon = 0.1$ (denoted by ``$\mu$") and  $K_u = 5$ (denoted by ``constant").
Setting $K_u = 5$ makes the algorithm about $20$ times faster than setting $\mu = 0.1$ and $100$--$600$ times faster than DP.
Note that robust ratios produced by our algorithm are greater than $OPT$ and regrets are smaller than $OPT$.
From the figures, we see that the (modified) RDP algorithm produces nearly optimal policy-parameter pairs.
In the rest of experiments, we do not show the results of the modified algorithm to solve problem~(\ref{eq:mmregret}).

We test on a larger network of $2028$ culverts and $166$ dams to see what value of $K$, when we set $K_u{=}K$, is sufficiently large for the RDP algorithm to produce good robust ratios.
The optimal objective value is not available on this network.
The results are shown in Fig.\,\ref{fig:appr_large}.
We see that robust ratios converge within 2 minutes for all testing policies, and random policies are much worst than two baseline policies.
The value of $K$ in the convergence area implies that it is sufficient to produce near-optimal solutions.

\begin{figure}[t]
\centering
\subfloat{\includegraphics[height=4mm, width=84mm]{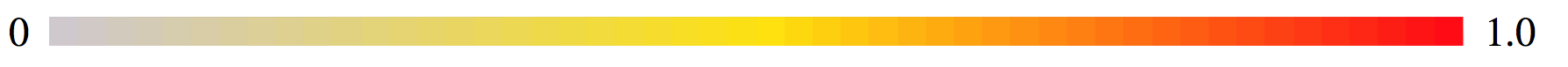}}\\
\vspace{-8pt}
\subfloat[\small ``midpoint" ($1.3 {\times} 10^6$)]{\includegraphics[height=40mm]{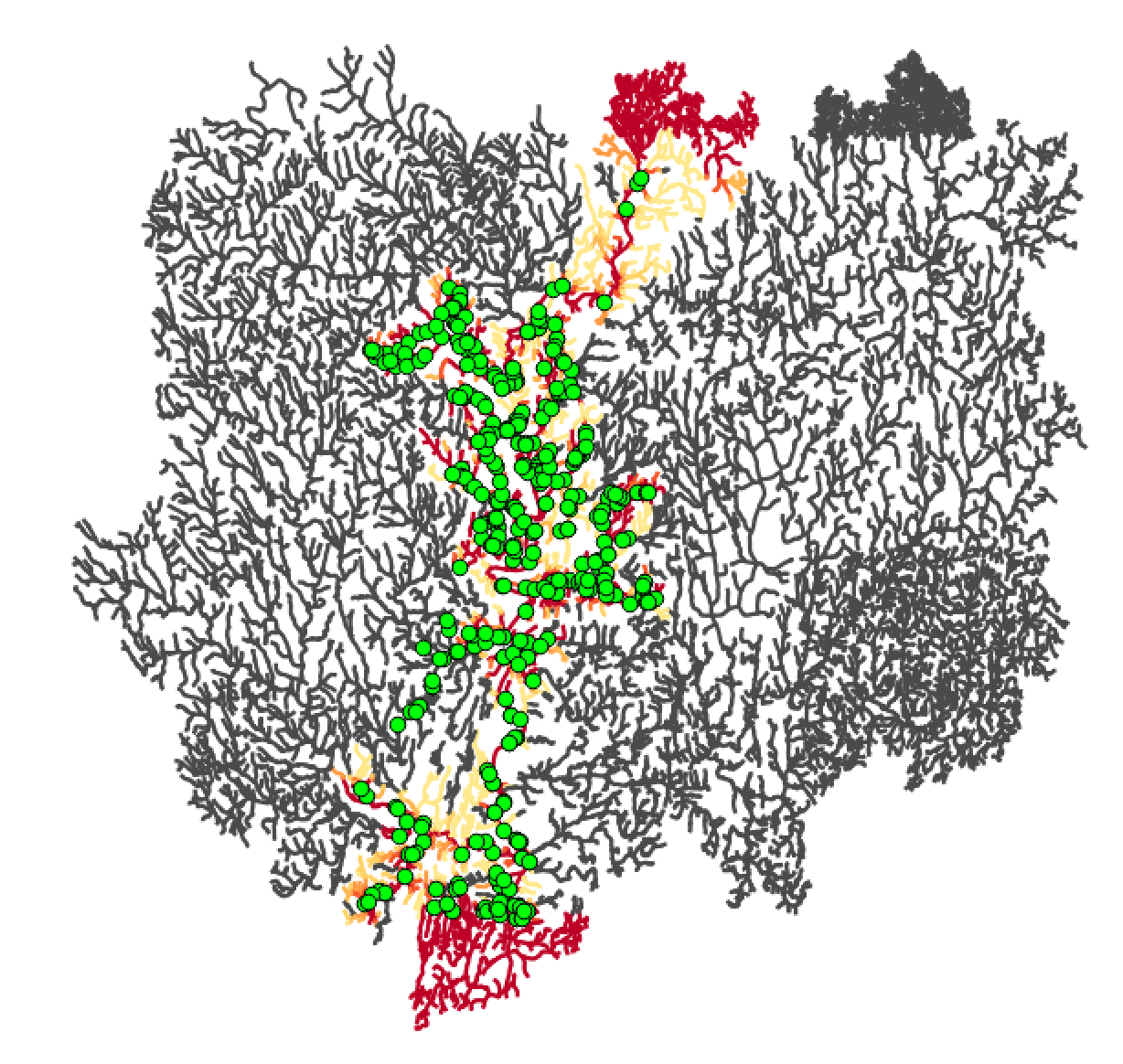}}
\subfloat[\small ``midpoint" adv.~($2.7 {\times} 10^6$)]{\includegraphics[height=40mm]{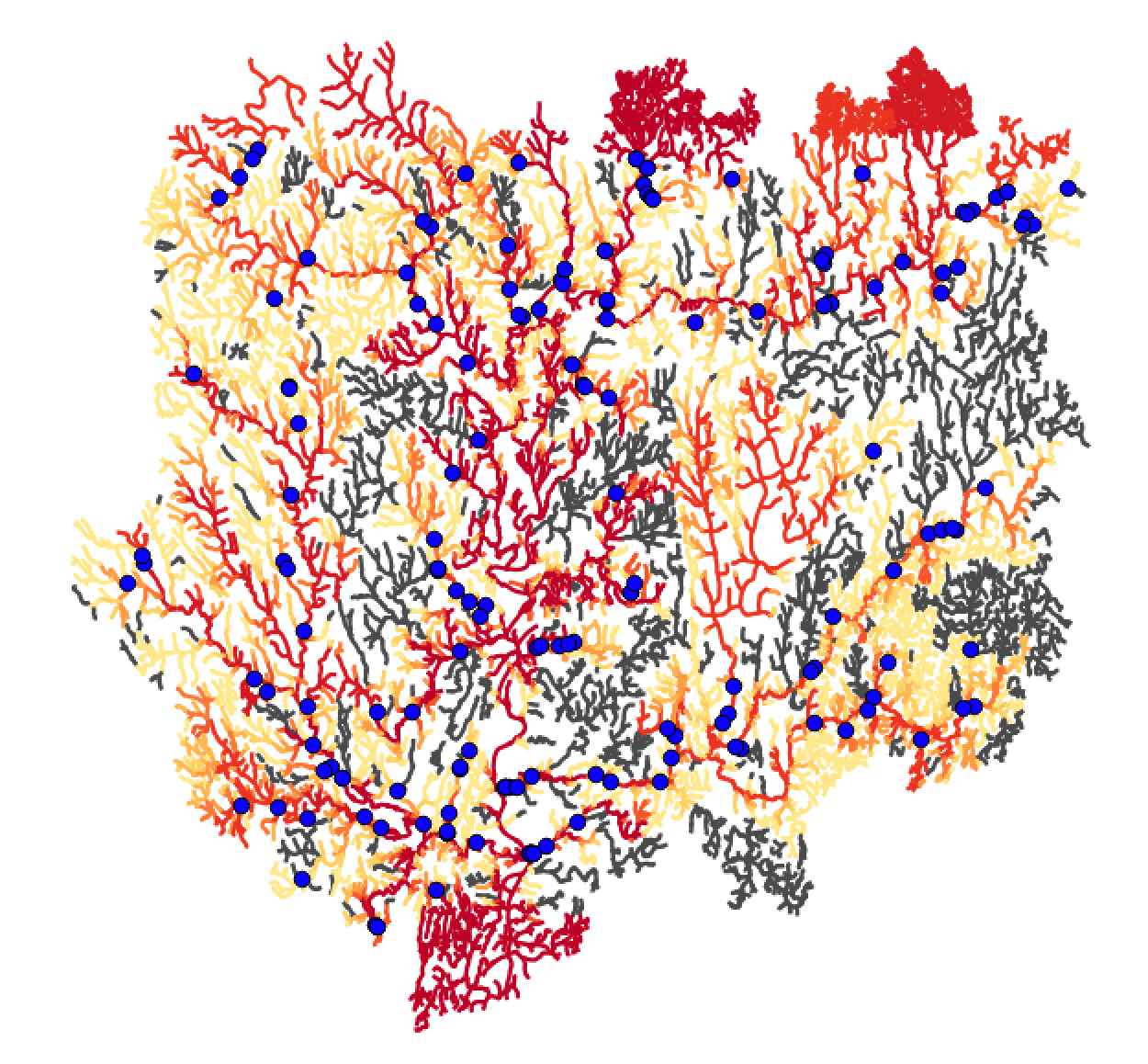}}\\
\subfloat[\small ``MRR" ($1.7 {\times} 10^6$)]{\includegraphics[height=40mm]{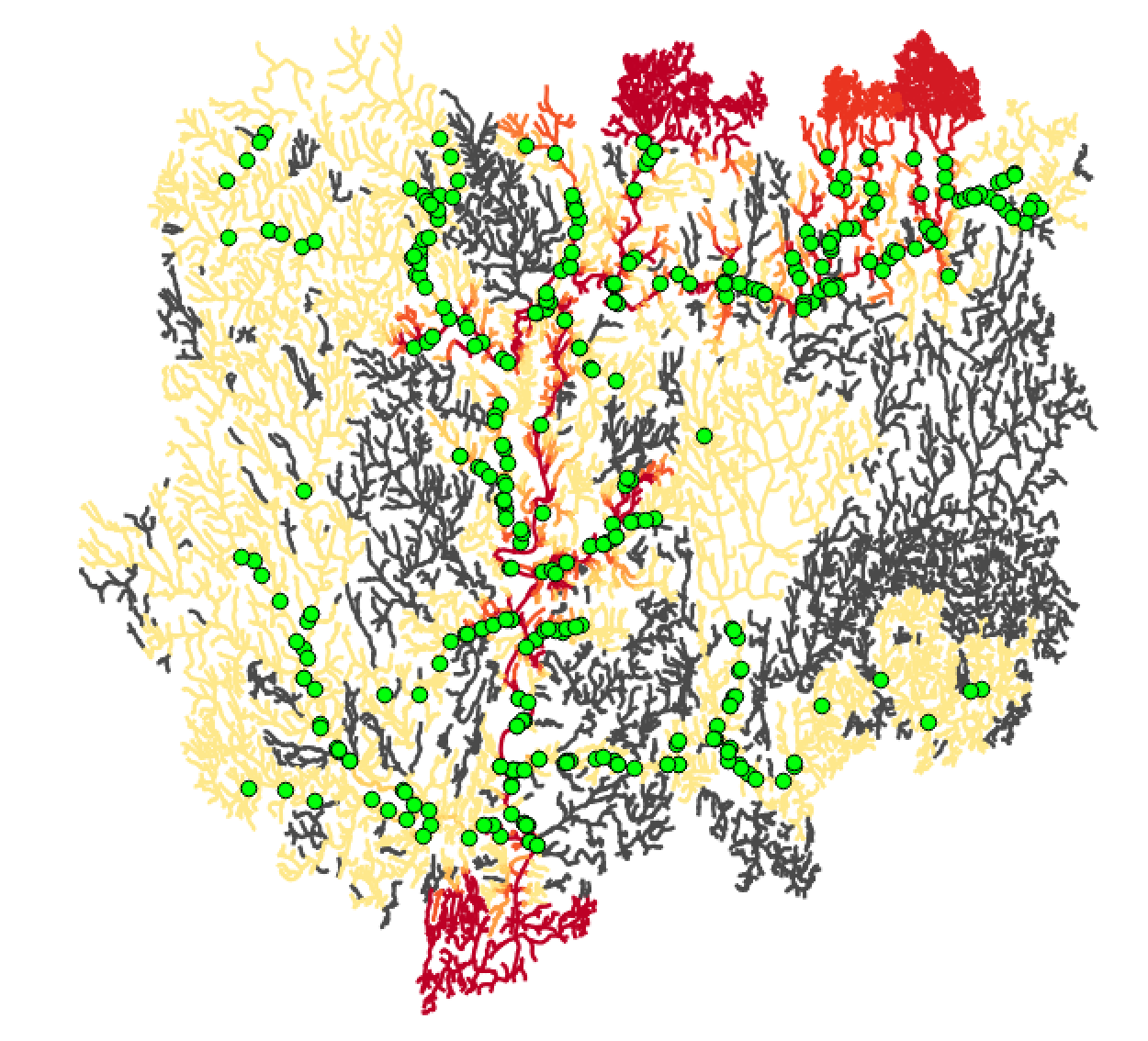}}
\subfloat[\small ``MRR" adv.~($2.2 {\times} 10^6$)]{\includegraphics[height=40mm]{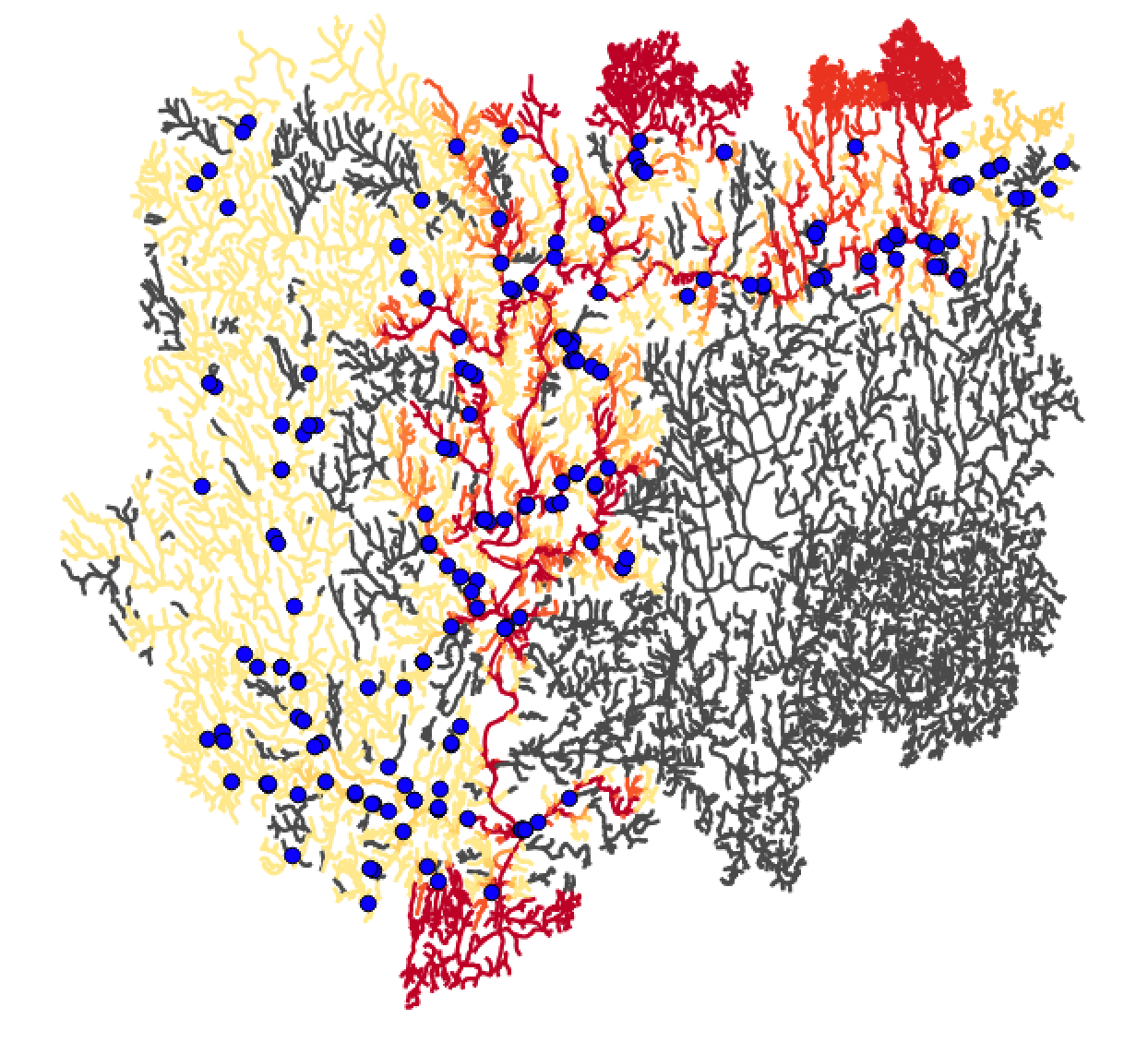}}
\caption{Visualization of four policies for $\beta{=}0.3$. Values shown in parenthesis. The accessibilities of edges are colored according to the top color bar. Dots represent removed (repaired) barriers. The adversarial midpoint and MRR policies are computed by RDP. }
\label{fig:visualize_policy_large}
\end{figure}

\subsubsection{Robustness Comparison}
On the same network, we compare the robustness of three policies using the value of $K$ in the convergence area.
Fig.\,\ref{fig:robust_ratio} shows how the robust ratio and regret computed by ``MRR" change as the size of intervals (i.e., $\beta$) varies.
Budget sizes are relative to the cost of removing all barriers.
We see that as $\beta$ increases, the robust ratio decreases and the regret increases almost linearly.
``MRR" gives the largest robust ratio.
Although ``MRR" maximizes the robust ratio, it produces the smallest regret, 
implying that the two robustness metrics are correlated.

Finally, we test our algorithms on a large network of $9335$ nodes, $7566$ culverts and $596$ dams with $5\%$ budget.
In this very difficult setting, we obtain results similar to those shown in Fig.\,\ref{fig:robust_ratio} even without using the value of $K$ in the convergence area.
Due to the limitation of space, we do not show those similar figures here, but only visualize the computed policies in Fig.\,\ref{fig:visualize_policy_large}.
The ``midpoint" policy allocates most of the budget around the main stream, near the middle vertical line of the river.
The adversarial policy can easily achieve much better value by taking actions in other important areas and assigns high probabilities if actions are taken (e.g., the adversarial policy) and low probabilities if actions are not taken (e.g., the decision policy.)
In contrast, the ``MRR" policy is more robust by allocating the budget to several important areas so that the adversarial policy cannot use the same trick to achieve much better value.

\section{Conclusion}
We describe an approximate robust optimization algorithm for a tree-structured stochastic network design problem, which is motivated by the river network design problem for fish conservation.
The algorithm iteratively solves two optimization problem: the decision optimization problem and the ratio minimization problem.
The former is encoded into a MILP, and an FPTAS is developed for the latter, which is the harder problem.
Empirically, we show that the policies computed by maximizing the robust ratio are more robust than policies computed by two other baseline methods.
Besides finding policies of high robust ratio, our algorithm can also produce policies with small regret on large-scale networks. 
These algorithms provide new computational tools for environmental scientists who tackle decision problems with imprecise models.

\section*{Acknowledgments}
This work was partially funded by a UMass Graduate School Dissertation Writing Fellowship awarded to the first author. Second author is supported by the research center at the School of Information Systems at the Singapore Management University.

\bibliographystyle{aaai} 
\bibliography{aaai2017}

\section*{Appendix}

\noindent\B{Lemma 1.} \emph{There exists an optimal policy-parameter pair $(\pi^{\prime*}, \bs{p}^*)$ for Problem~(\ref{eq:step3}) with the following property.
Suppose $\pi^{\prime*}$ takes action $i$ and the decision policy $\pi$ takes action $j$ on edge $e$.
If $j\neq i$, then $p^*_{e|i} = \ub{p}_{e|i}$ and $p^*_{e|j} = \lb{p}_{e|j}$.
Otherwise, $p^*_{e|i}$ is either $\lb{p}_{e|i}$ or $\ub{p}_{e|i}$.}
\begin{proof}
Suppose $\pi^{\prime *}$ takes action $i$ and the decision policy $\pi$ takes action $j$ on edge $e$.
Let us first consider the case when $j \neq i$.
Since only $\pi$ takes action $j$ on $e$, the term $p_{e|j} $ only appears in the numerator of the robust ratio in (4), and $p_{e|i}$ only appears in the denominator.
Then, the minimization w.r.t. $\bs{p}$ in (4) can be written as
\begin{align*}
\min_{\bs{p}} \frac{a p_{e|j} + b}{c p_{e|i} + d}
\end{align*}
where $a,b,c,d$ are constants w.r.t. $p_{e|i}$ and $p_{e|j}$.
Since all rewards and probabilities are nonnegative, coefficients $a$ and $c$ are nonnegative. The optimal probability setting will satisfy $p^*_{e|i} = \ub{p}_{e|i}$ and $p^*_{e|j} = \lb{p}_{e|j}$, which proves the first part of the lemma.

Let us consider the case when $j = i$.
Now, $p_{e|i}$ will appear in both numerator and denominator.
In this case, we have
{\small \begin{align*}
& \min_{\bs{p}} \frac{a p_{e|i} + b}{c p_{e|i} + d} \\
= & \min_{\bs{p}}\frac{\frac{a}{c}(c p_{e|i} + d) + b - \frac{ad}{c}}{c p_{e|i} + d} \\
= & \min_{\bs{p}}\frac{a}{c} + \frac{b - \frac{ad}{c}}{c p_{e|i} + d}
\end{align*}}
where $a,b,c,d$ are nonnegative constants w.r.t. $p_{e|i}$.
If $b \geq \frac{ad}{c}$, the optimal probability setting will set $p_{e|i} = \ub{p}_{e|i}$. Otherwise, it will set $p_{e|i} = \lb{p}_{e|i}$. 
In summary, $p^*_{e|i}$ is either the upperbound or the lowerbound, which proves the second part.
\end{proof}

\vspace{10pt}
\noindent\B{Corollary 1.} \emph{For a fixed $\pi$, only $|A_e| + 1$ actions in $A^s_e$ are needed to compute $(\pi^{\prime*}, \bs{p}^*)$.} 
\begin{proof}
Due to Lemma 1, if $\pi'$ and $\pi$ take different actions, there is only one possible probability setting that we need to consider.
If they take the same action (say $i$), there are two cases $p_{e|i} = \lb{p}_{e | i}$ or $p_{e|i} = \ub{p}_{e | i}$ while the probabilities of other actions than $i$ can be chosen arbitrarily and don't affect the objective value of both the decision policy and the adversarial policy.
\end{proof}

\vspace{10pt}
\noindent\B{Theorem 4. } \emph{$T(n_u) = O(\frac{n_u^4}{\mu^2})$.}
\begin{proof}
We have
{\small \begin{align*}
T(n_u) &= O(m^a_v m^d_v m^a_w m^d_w) + T(n_v) + T(n_w) \\
& \leq c\frac{n_v^2 n_w^2}{\mu^2} + T(n_v) + T(n_w) \\
& \leq \max_{0\leq k \leq n_u - 1} c\frac{k^2 (n_u - k - 1)^2}{\mu^2} + T(k) + T(n_u - k - 1)
\end{align*}
}
where $n_u, n_v, n_w$ are the numbers of nodes in subtree at $u,v,w$ and $n_u = n_v + n_w + 1$.

To show that $T(n_u) = O(\frac{n_u^4}{\mu^2})$, we use induction.
For the base case, the DP table at a leave node has only one tuple, so $T(1) = O(1) = O(\frac{1}{\mu^2})$ as $\mu < 1$.
To do the induction, let $v$ and $w$ be the two children of $u$ and assume that $T(n_v) = c_1\frac{n^4_v}{\mu^2}$ and $T(n_w) = c_2\frac{n^4_w}{\mu^2}$.
Let $c' = \max\{c_1, c_2, c\}$ where $c$ is the constant in previous inequalities. Continuing the derivation of $T(n_u)$, we have
{\small 
\begin{align*}
T(n_u) &\leq \frac{c'}{\mu^2}\max_{0\leq k \leq n_u - 1} 2k^2 (n_u - k - 1)^2 + k^4 + (n_u - k - 1)^4 \\
& \leq \frac{c'}{\mu^2}\max_{0\leq k \leq n_u - 1} (k^2 + (n_u - k - 1)^2)^2 \\
& \leq \frac{c'}{\mu^2}\max_{0\leq k \leq n_u - 1} (k^2 + 2k (n_u - k - 1) + (n_u - k - 1)^2)^2 \\
& \leq \frac{c' n_u^4}{\mu^2}
\end{align*} 
}
Thus, we have shown that $T(n_u) = O(\frac{n^4_u}{\mu^2})$.
\end{proof}

\end{document}